\documentclass[runningheads]{llncs}
\usepackage[T1]{fontenc}
\usepackage{graphicx}
\usepackage{graphics} 
\usepackage{epsfig} 
\usepackage{amsmath} 
\usepackage{amssymb}  
\usepackage{mathtools}
\usepackage{tensor}    
\usepackage{xcolor}
\usepackage{booktabs} 
\usepackage{multirow}
\usepackage{colortbl}
\usepackage{BOONDOX-cal} 

\usepackage{algorithm} 
\usepackage{algpseudocode}

\usepackage{cite} 

\usepackage{hyperref}
\hypersetup{
	colorlinks=true,
	urlcolor=blue
}

\newcommand{\rot}[2]{\tensor*[ ^{#2}_{#1} ]{\boldsymbol{R}}{}} 

\newcommand{\pos}[2]{\tensor*[_{#1\hspace{1pt}}^{#2}]{\boldsymbol{p}}{}} 

\newcommand{\s}[1]{\tensor*[_{#1}]{\boldsymbol{\mathcal{\xi}}}{}} 

\newcommand{\skewsym}[1]{   \left[{#1}\right]_\times}    

\newcommand{\refframe}[1]{\{#1\}}

\newcommand{\mat}[1]{\boldsymbol{#1}} 
\newcommand{\screw}[1]{\boldsymbol{\mathcal{#1}}} 

\begin{document}
	
\title{A Coordinate-Invariant Local Representation of Motion and Force Trajectories for Identification and Generalization Across Coordinate Systems}

\titlerunning{Coordinate-Invariant Representation of Motion and Wrench Trajectories}

\author{Arno~Verduyn\inst{1,2,}\thanks{Corresponding author: \email{arno.verduyn@kuleuven.be}}\orcidID{0000-0002-5073-3881} 
\and
Erwin~Aertbeli\"en\inst{1,2}\orcidID{0000-0002-4514-0934} 
\and
Maxim~Vochten\inst{3}\orcidID{0000-0001-5070-846X} 
\and
Joris~De~Schutter\inst{2}\orcidID{0000-0001-9619-5815}}
\authorrunning{A. Verduyn et al.}
%
\institute{These authors contributed equally to this work. \and Department of Mechanical Engineering and Flanders Make at KU Leuven, 3001 Leuven, Belgium. \and
Department of Mechanics, Royal Military Academy, 1000 Brussels, Belgium.}
\maketitle              
\begin{abstract}

Identifying the trajectories of rigid bodies and of interaction forces is essential for a wide range of tasks in robotics, biomechanics, and related domains. These tasks include trajectory segmentation, recognition, and prediction. For these tasks, a key challenge lies in achieving consistent results when the trajectory is expressed in different coordinate systems.
A way to address this challenge is to utilize trajectory models that can generalize across coordinate systems. The focus of this paper is on such trajectory models obtained by transforming the trajectory into a coordinate-invariant representation.
However, coordinate-invariant representations often suffer from sensitivity to measurement noise and the manifestation of singularities in the representation, where the representation is not uniquely defined. This paper aims to address this limitation by introducing the novel Dual-Upper-Triangular Invariant Representation (DUTIR), with improved robustness to singularities, along with its computational algorithm. The proposed representation is formulated at a level of abstraction that makes it applicable to both rigid-body trajectories and interaction-force trajectories, hence making it a versatile tool for robotics, biomechanics, and related domains.

\keywords{Trajectory \and Invariance \and Screw Theory \and  Twist \and Wrench.}
\end{abstract}
\section{Introduction}
\label{sec:intro}
Identifying the trajectories of rigid bodies and of interaction forces is essential for a wide range of tasks in robotics~\cite{vochten2023invariant}, biomechanics~\cite{Ancillao2022}, and related fields. These tasks include trajectory segmentation~\cite{verduyn2023enhancingmotiontrajectorysegmentation}, recognition~\cite{WU2008,Joris2012,Vochten2015,DSRF2018,Lee2018,RRV2018}, and prediction~\cite{prediction_dual_quat}. Typically, these tasks are performed using algorithms that operate on trajectory coordinates expressed in a specific coordinate system, such as a body-fixed coordinate system attached to a mobile robot, or a coordinate system attached to a sensor. 
A key challenge, however, is to achieve consistent results across different choices of coordinate system~\cite{vochten2019generalizing}. This challenge arises whenever the context of the trajectory changes, for instance, due to a new robot configuration or a different spatial location or orientation of the trajectory.
A common way to avoid this challenge is to enforce consistent coordinate systems during data collection through coordinate system calibration procedures.
However, these calibration procedures are often labour-intensive. 
Furthermore, an open problem remains: how to handle existing datasets in which the coordinate systems are misaligned and where the relationships between the different coordinate systems are unknown. A more rigorous way to address this challenge is to learn trajectory models that generalize across coordinate systems. In the literature, methods exist to learn such models from demonstration data \cite{DeSchutterJoris2010,Vochten2015,vochten2023invariant,DSRF2018,Lee2018,RRV2018,CNN2016,ViewInvariantANN2012}. A differentiator among these methods is the amount of data that is required. In this paper, we focus on methods that require only one demonstration. Specifically, we focus on methods based on linear algebra and differential geometry that transform the demonstrated trajectory into a representation that remains unchanged when the coordinate system changes, referred to as a \textit{coordinate-invariant representation}.

Several invariant representations have been introduced in the literature. Some works utilize singular value decompositions to embed trajectory segments into coordinate-invariant subspaces~\cite{RRV2018}. Other works~\cite{DeSchutterJoris2010,Vochten2015,vochten2023invariant,Lee2018,RRV2018} aim to eliminate the need for prior trajectory segmentation by proposing representations based on \textit{local} differential-geometric trajectory features that are inherently coordinate-invariant. Such local features can be computed either analytically in closed form~\cite{DeSchutterJoris2010,Vochten2015} or iteratively via optimization-based methods~\cite{vochten2023invariant}.

Despite the advantages of pure locality and coordinate-invariance, local invariant representations often show low robustness in practice due to sensitivity to measurement noise and the manifestation of singularities in the representation, where the representation is not uniquely defined. Furthermore, the iterative optimization-based method proposed in~\cite{vochten2023invariant} requires relatively long computation times compared to analytical approaches, which limits its suitability for real-time use. These limitations motivate the research question of the work in this paper:
\begin{center}
	\textit{How can rigid-body trajectories and interaction-force trajectories be represented in a coordinate-invariant manner while ensuring (1) pure locality of the representation, (2) low computational complexity, and (3) high robustness?
	}
\end{center} 
This paper introduces a novel local invariant trajectory representation with improved robustness to singularities, along with its analytical computational algorithm. An earlier version of this work appeared as a preprint~\cite{verduyn2025biltsbiinvariantsimilaritymeasure}. Since then, the invariant representation has been successfully applied to motion recognition tasks and shown to outperform other invariant representations~\cite{phdarno,HPM_recognition}. 
However, the theoretical foundations of the invariant representation have not yet been formally introduced in a peer-reviewed publication. The present paper aims to fill this gap. Furthermore, in this paper, the proposed representation is reformulated at a level of abstraction that makes it applicable to both rigid-body trajectories and interaction-force trajectories, hence making it a versatile tool for robotics, biomechanics, and related domains. 

\section{Preliminaries}
\label{sec:preliminaries}
This section explains the concepts of \textit{twist}, \textit{wrench} and the general \textit{screw}, which are mathematical entities that can be used to represent rigid-body motions, interaction forces, and their \textit{trajectories}. 

\subsection{Rigid-Body Motion}
A \textit{body} can be interpreted as a collection of points. If the relative distances between these points remain constant over time, the body is considered non-deformable or \textit{rigid}. 
To describe its motion, we commonly define a coordinate system and rigidly attach it to the body. The orientation of this coordinate system then represents the orientation of the body, while the position of its origin represents the position of a chosen point of the body, referred to as a \textit{body reference point}.
The motion of the rigid body can then be completely described by the kinematics of this coordinate system attached to the body. The first-order kinematics consist of the rotational velocity $\boldsymbol{\omega}$ of the coordinate system and the translational velocity $\boldsymbol{v}$ of its origin. 

The two three-dimensional vectors, $\boldsymbol{\omega}$ and $\boldsymbol{v}$, are commonly combined into a single six-dimensional vector, $\screw{t} = (\boldsymbol{\omega}^\top, \boldsymbol{v}^\top)^\top$, referred to as the \textit{twist}.
Numerically, the twist can be computed from the position and orientation of the body-fixed coordinate system at two successive time instances using the logarithmic map~\cite{lynch2017modern}. In this paper, we only consider twists where the coordinates of both $\boldsymbol{\omega}$ and $\boldsymbol{v}$ are expressed in one and the same coordinate system, and $\boldsymbol{v}$ represents the velocity of the origin of that coordinate system.
Such a twist is also referred to as a \textit{screw}\cite{murray1994}, as further explained in Section~\ref{sec:screws}.

We indicate the choice of coordinate system for the twist with a subscript. For example, ${_a}\screw{t}$ is read as ``the twist expressed in coordinate system~$\{a\}$.''
Changing between coordinate systems, for example from $\refframe{a}$ to $\refframe{b}$ can be done using the $6\times6$ \textit{transformation matrix} ${^a_b}\mat{S}$ :
\begin{align}
	\label{eq:screw_transformation}
	_b\screw{t} & = {^a_b}\mat{S} \ _a\screw{t} \quad \text{with} \quad {^a_b}\mat{S} =  \small \begin{bmatrix}
		\rot{b}{a}            & \boldsymbol{0} \\
		\skewsym{\pos{b}{a}} \rot{b}{a}  & \rot{b}{a}
	\end{bmatrix},
\end{align}
and where $\rot{b}{a}$ represents the relative orientation of $\{a\}$ with respect to $\{b\}$, encoded as a $3\times3$ orthogonal rotation matrix, and where $\pos{b}{a}$ represents the position vector from the origin of $\{b\}$ to the origin of $\{a\}$, with its coordinates expressed in coordinate system $\{b\}$.  The operator $[~\cdot~]_\times$ in \eqref{eq:screw_transformation} transforms the three-dimensional vector $\pos{b}{a}$ into a $3\times3$ skew-symmetric matrix, which is commonly used to represent cross-product operations in matrix form~\cite{murray1994}.

\subsection{Rigid-Body Wrench}

A rigid body can be subjected to forces and torques. In that case, all the forces and torques acting on the body can be reduced to a single resultant force $\boldsymbol{f}$ and a single resultant torque $\boldsymbol{\tau}$ about a specified reference point without loss of generality. 
The two three-dimensional vectors $\boldsymbol{f}$ and $\boldsymbol{\tau}$ can be combined into a six-dimensional vector $\screw{w} = (\boldsymbol{f}^\top, \boldsymbol{\tau}^\top)^\top$, referred to as the \textit{wrench}.
Similarly to twists, we again only consider wrenches with their components, $\boldsymbol{f}$ and $\boldsymbol{\tau}$, corresponding to a single coordinate system.
Changing between coordinate systems can be done similarly as for twists using the same $6\times6$ transformation matrix in \eqref{eq:screw_transformation}, thanks to the chosen ordering of the components $\boldsymbol{f}$ and $\boldsymbol{\tau}$ in the wrench.

\subsection{Abstraction to General Screws}
\label{sec:screws}
Both the twist and the wrench belong to the category of \textit{screws}. A screw is a six-dimensional vector that can be associated with a physically meaningful axis in space, referred to as the underlying \textit{screw axis}. The interpretation of screw axes for twists and wrenches originates from classical mechanics, with Chasles~\cite{chasles1830note} showing that a rigid-body motion can always be described as a rotation and translation along an axis in space, and with Poinsot~\cite{poinsot} showing that force systems can always be reduced to a single force and torque along an axis. 

Formally, a screw, denoted as $\boldsymbol{\xi}$, is a six-dimensional vector consisting of two vectors, $\boldsymbol{\alpha}\in\mathbb{R}^{3\times1}$ and  $\boldsymbol{\beta}\in\mathbb{R}^{3\times1}$:
\begin{equation}
	\boldsymbol{\xi} =  \small \begin{pmatrix} \boldsymbol{\alpha} \\ \boldsymbol{\beta}
	\end{pmatrix} . 
\end{equation} 
For twists $\boldsymbol{\alpha} = \boldsymbol{\omega}$ and $\boldsymbol{\beta} = \boldsymbol{v}$, while for wrenches $\boldsymbol{\alpha} = \boldsymbol{f}$ and $\boldsymbol{\beta} = \boldsymbol{\tau}$.
The direction of the screw axis of $\boldsymbol{\xi}$, denoted as $\boldsymbol{e}$, can be found by:
\begin{equation}
	\label{eq:direction_screw}
	\boldsymbol{e} =  \small \frac{\boldsymbol{\alpha}}{\Vert\boldsymbol{\alpha}\Vert},
\end{equation} 
while the position of the point $\boldsymbol{p}_\perp$ on the screw axis closest to the origin of the coordinate system can be found by: 
\begin{equation}
	\label{eq:p_perp}
	\boldsymbol{p}_\perp =  \small \frac{\boldsymbol{\alpha}\times\boldsymbol{\beta}}{\Vert\boldsymbol{\alpha}\Vert^2}. 
\end{equation} 

Note that all screws $\boldsymbol{\xi}$ transform between coordinate systems with the transformation matrix $\boldsymbol{S}$ as defined in \eqref{eq:screw_transformation}. For this reason, the matrix $\boldsymbol{S}$ is also referred to as the \textit{screw-transformation matrix}. 

In the remainder of this paper, we reason at the level of general screws. As a result, the introduced concepts and methods apply seamlessly to representations of rigid-body motion and of interaction forces.

\subsection{Screw Trajectories}
To analyse how a screw evolves during a certain time duration, we can evaluate the screw at consecutive time instances. We denote the time instance at which the screw $\boldsymbol{\xi}$ is evaluated explicitly by adding the argument $[t_i]$, yielding $\boldsymbol{\xi}[t_i]$.
In this specific case, time is used to represent a certain \textit{progress variable} along which the screw is evaluated. 
We refer to the sequence of screws: $\boldsymbol{\xi}[t_i]$ for $i\in[1,N]$ as a \textit{discrete-time screw trajectory}, where $N$ represents the total number of time samples. The choice for a discrete trajectory is motivated by the fact that this is how trajectory data is commonly provided by sensors in practical applications.

Apart from time, other choices for the progress variable exist. For example, when analysing motion trajectories of rigid bodies, common choices for the progress variable are the angle traversed by the body during rotation~\cite{Roth2005}, and the arc length traced by a reference point attached to the body during translation~\cite{DSRF2018}. In the remainder of this paper, we make abstraction of the specific choice for the progress variable and reason at the level of a generic progress variable denoted by $x$, yielding $\boldsymbol{\xi}[x_i]$.
This unified perspective allows the introduced concepts and methods to apply seamlessly to all choices for the progress variable.

\section{Proposed Method}
\label{sec:method}

As explained in the previous section, the coordinates of screw trajectories depend on the choice of coordinate system in which they are expressed. The aim of this section is to introduce a novel coordinate-invariant representation of screw trajectories. The proposed representation is obtained by first capturing the local evolution of the trajectory (see Sec.~\ref{sec:local_rep}) and then describing this local trajectory evolution in a coordinate-invariant manner (see Sec.~\ref{sec:SU}).

\subsection{Representation of Local Trajectory Evolution}
\label{sec:local_rep}

Given that the second-order derivative of the screw $\s{}[x_i]$ is numerically approximated by:
\begin{equation}
	\s{}''[x_i] =  \small \frac{\s{}[x_{i+1}] - 2\s{}[x_i] + \s{}[x_{i-1}]}{(\Delta x)^2} + \mathcal{O}(\Delta x),
\end{equation}
it follows that the sequence of the three consecutive screws: $\s{}[x_{i-1}]$, $\s{}[x_i]$, and $\s{}[x_{i+1}]$, captures both the magnitude and direction of $\s{}[x_i]$, as well as its local second-order evolution at $x_i$. We denote this sequence of screws by the $6\times3$ matrix $\boldsymbol{\Xi}[x_i]$:
\begin{equation}
	\label{eq:def_Xi}
	\boldsymbol{\Xi}[x_i] = \begin{bmatrix}
		\s{}[x_{i-1}] & \s{}[x_i] & \s{}[x_{i+1}]
	\end{bmatrix} =  \small \begin{bmatrix}
	\boldsymbol{\alpha}[x_{i-1}] & \boldsymbol{\alpha}[x_i] & \boldsymbol{\alpha}[x_{i+1}] \\
	\boldsymbol{\beta}[x_{i-1}] & \boldsymbol{\beta}[x_i] & \boldsymbol{\beta}[x_{i+1}]
	\end{bmatrix}.
\end{equation}
Further on, we omit the argument $[x_i]$ when clear from the context. 

\subsection{Coordinate-Invariant Representation of $\boldsymbol{\Xi}$}
\label{sec:SU}
We will show that the matrix $\boldsymbol{\Xi}$ can be mathematically decomposed into the product of a screw-transformation matrix $\boldsymbol{S}$ and a coordinate-invariant representation $\boldsymbol{U}$. We refer to this novel decomposition as the \textit{SU-decomposition}\footnote{The proposed decomposition is inspired by the QR-decomposition of a $3\times3$ matrix.} and to the corresponding representation as the \textit{Dual-Upper-Triangular Invariant Representation} (DUTIR).
Theorem~\ref{th:SU} introduces these concepts, while Property~\ref{pr:SU_invariance} details the coordinate invariance of the corresponding representation.

\begin{theorem}[Existence of an $SU$-Decomposition]
	\label{th:SU}
	For any $\boldsymbol{\Xi}[x_i]$ in $\mathbb{R}^{6\times3}$ that satisfies the two regularity conditions:
	\begin{equation}
		\label{eq:SU_reg_con_1}
		\Vert \boldsymbol{\alpha}[x_{i-1}] \Vert \neq 0, \quad \text{and} \quad \boldsymbol{\alpha}[x_{i-1}] \times \boldsymbol{\alpha}[x_{i}] \neq \boldsymbol{0},
	\end{equation}
	there exists a unique $6\times6$ screw-transformation matrix $\boldsymbol{S}$ and $6 \times 3$ twice upper triangular matrix $\boldsymbol{U}$, such that:
	\begin{equation}
		\label{eq:bilts_SU}
		\boldsymbol{\Xi} = \boldsymbol{S} \ \boldsymbol{U}, 
	\end{equation}
	where $\boldsymbol{U}$ is obtained by stacking two \(3\times 3\) upper-triangular blocks
	$\boldsymbol{U}_1$ and $\boldsymbol{U}_2$:
	\begin{equation}
		\boldsymbol{U} =
		\begin{bmatrix}
			\boldsymbol{U}_1\\[2pt]
			\boldsymbol{U}_2
		\end{bmatrix}
		=  \small
		\begin{bmatrix}
			u_{11} & u_{12} & u_{13} \\
			0      & u_{22} & u_{23} \\
			0      & 0      & u_{33} \\
			u_{41} & u_{42} & u_{43} \\
			0      & u_{52} & u_{53} \\
			0      & 0      & u_{63}
		\end{bmatrix},
	\end{equation}
	and where the convention $u_{11}>0$ and $u_{22}>0$ is imposed to ensure uniqueness. 
\end{theorem}

\begin{proof}
Decomposition~\eqref{eq:bilts_SU} can be expanded by writing $\boldsymbol{\Xi}$ as a vertical stack of two $3\times3$ matrices, $\boldsymbol{A}$ and $\boldsymbol{B}$, and by using the definition of $\boldsymbol{S}$ in \eqref{eq:screw_transformation}:
\begin{equation}
	\boldsymbol{\Xi}  =  \small
	\begin{bmatrix}
		\boldsymbol{A} \\
		\boldsymbol{B}
	\end{bmatrix}
	= \begin{bmatrix}
		~\rot{}{}~            & ~\boldsymbol{0}~ \\
		~\skewsym{\pos{}{}}\hspace{-2pt}\rot{}{}~  & ~\rot{}{}~
	\end{bmatrix}
	\begin{bmatrix}
		\boldsymbol{U}_1 \\
		\boldsymbol{U}_2
	\end{bmatrix} .
	\label{eq:extendedQR}
\end{equation}
The first block row of \eqref{eq:extendedQR} can be expressed as follows:
\begin{equation}
	\boldsymbol{A} = \rot{}{} \  \boldsymbol{U}_1 =  \small \rot{}{} \begin{bmatrix}
		u_{11} & u_{12} &  u_{13}\\
		0 & u_{22} & u_{23} \\
		0 & 0 & u_{33} \\
	\end{bmatrix},
	\label{eq:A1}
\end{equation}
where the first two columns of $\rot{}{}$ and $\boldsymbol{U}_1$ can be found using a standard QR-decomposition algorithm~\cite{golub2013matrix}, which is well-defined if the first two diagonal elements of $\boldsymbol{U}_1$, $u_{11}$ and $u_{22}$, both differ from zero. Afterwards, the sign convention, $u_{11}>0$ and $u_{22}>0$, can be imposed by flipping the sign of the corresponding row of $\boldsymbol{U}_1$ and of the corresponding column of $\rot{}{}$ whenever necessary.
The third column of $\rot{}{}$ and the sign of the third diagonal element $u_{33}$ is then determined by imposing that $\rot{}{}$ is a right-handed orthogonal matrix. This procedure results in a unique solution for the rotation matrix $\rot{}{}$ and upper-triangular matrix $\boldsymbol{U}_1$. 

The conditions for non-zero diagonal elements, $u_{11} \neq 0$ and $u_{22} \neq 0$, result in the regularity conditions shown in~\eqref{eq:SU_reg_con_1}. This follows from the definition of the QR-decomposition. Specifically, the first two columns of $\boldsymbol{R}$, i.e. $\boldsymbol{r}_1$ and $\boldsymbol{r}_2$, are related to $\boldsymbol{\alpha}[x_{i-1}]$ and $\boldsymbol{\alpha}[x_i]$ as follows:
\begin{align}
	\label{eq:r1}
	\boldsymbol{r}_{1} = \small \frac{\boldsymbol{\alpha}[x_{i-1}]}{\Vert \boldsymbol{\alpha}[x_{i-1}] \Vert}, \quad \text{and} \quad
	\boldsymbol{r}_{2} = \small \frac{\left(\boldsymbol{\alpha}[x_{i-1}] \times \boldsymbol{\alpha}[x_{i}]\right)\times\boldsymbol{\alpha}[x_{i-1}]}{\Vert \left(\boldsymbol{\alpha}[x_{i-1}] \times \boldsymbol{\alpha}[x_{i}]\right)\times\boldsymbol{\alpha}[x_{i-1}] \Vert}. 
\end{align}
Using~\eqref{eq:r1}, \eqref{eq:A1}, and \eqref{eq:def_Xi}, the scalars $u_{11}$, $u_{12}$ and $u_{22}$ can be written as:
\begin{equation}
	\label{eq:u_expanded}
	u_{11} = \boldsymbol{\alpha}[x_{i-1}] \cdot \boldsymbol{r}_{1} \quad,\quad
	u_{12} = \boldsymbol{\alpha}[x_{i}] \cdot \boldsymbol{r}_{1} \quad,\quad \text{and} \quad
	u_{22} = \boldsymbol{\alpha}[x_{i}] \cdot \boldsymbol{r}_{2}.
\end{equation}
Hence, $u_{11}$ is well-defined and non-zero if $\Vert\boldsymbol{\alpha}[x_{i-1}]\Vert \neq 0$, while $u_{22}$ is well-defined and non-zero if additionally $\boldsymbol{\alpha}[x_{i-1}] \times \boldsymbol{\alpha}[x_{i}] \neq \boldsymbol{0}$.

Thus far, matrices $\rot{}{}$ and $\boldsymbol{U}_1$ are completely and uniquely determined. The position vector $\pos{}{}$ in \eqref{eq:extendedQR} can then be determined as follows. After pre-multiplication with $\rot{}{}^\top$, the lower block row of \eqref{eq:extendedQR} can be written as:
\begin{align}
	\label{eq:1}
	\rot{}{}^\top \boldsymbol{B} & = \left( \rot{}{}^\top \skewsym{\pos{}{}} \rot{}{}\right) \  \boldsymbol{U}_1 + \boldsymbol{U}_2 .
\end{align}
Equation \eqref{eq:1} can be rewritten using the orthogonal basis transformation property of a skew-symmetric matrix~\cite{murray1994}: $\rot{}{}^\top [\pos{}{}]_\times \rot{}{} = [\rot{}{}^\top \pos{}{}]_\times$. This results in:
\begin{align}
	\rot{}{}^\top \boldsymbol{B} & = \left( [\rot{}{}^\top \pos{}{} \hspace{2pt}]_\times \right) \  \boldsymbol{U}_1 + \boldsymbol{U}_2.
\end{align}
Introducing $\pos{}{}^* \coloneq \rot{}{}^\top \pos{}{}$ results in:
\begin{align}
	\rot{}{}^\top \boldsymbol{B} & = \skewsym{\pos{}{}^*} \boldsymbol{U}_1 + \boldsymbol{U}_2.
	\label{eq:extendqr-origin}
\end{align}
From the lower-diagonal entries in \eqref{eq:extendqr-origin}, three scalar equations are obtained that do not rely on the unknown upper-triangular elements in $\boldsymbol{U}_2$. 
\begin{align}
	\label{eq:extendedqr_scalar0}
	(\rot{}{}^\top \boldsymbol{B})_{21} &= u_{11} \hspace{1pt} p^*_z, \\
	\label{eq:extendedqr_scalar1}
	(\rot{}{}^\top \boldsymbol{B})_{31} &= -u_{11} \hspace{1pt} p^*_y, \\          
	(\rot{}{}^\top \boldsymbol{B})_{32} &= u_{22} \hspace{1pt}  p^*_x -u_{12} \hspace{1pt} p^*_y, 	\label{eq:extendedqr_scalar}
\end{align}
where the notation $(\rot{}{}^\top \boldsymbol{B})_{ij}$ in the left-hand side of equations \eqref{eq:extendedqr_scalar0}-\eqref{eq:extendedqr_scalar} denotes the element in the $i^\text{th}$ row and $j^\text{th}$ column of $\rot{}{}^\top \boldsymbol{B}$. 

The coordinates $p^*_x,~p^*_y,~p^*_z$ of $\pos{}{}^*$ can be uniquely determined from  \eqref{eq:extendedqr_scalar0}-\eqref{eq:extendedqr_scalar}, again under the same conditions: $u_{11}\neq0$ and $u_{22}\neq0$.
Now that $\rot{}{}$, $\boldsymbol{U}_1$, and $\pos{}{}^*$ are completely determined, $\pos{}{}$ can be found as $\pos{}{} = \rot{}{} \ \boldsymbol{p}^*$, while $\boldsymbol{U}_2$ can be solved from~\eqref{eq:extendqr-origin}:
\begin{align}
	\label{eq:R2-computation}
	\boldsymbol{U}_2 & = \rot{}{}^\top \boldsymbol{B} - \skewsym{\pos{}{}^*} \boldsymbol{U}_1 .
\end{align}
The upper-triangular matrices $\boldsymbol{U}_1$ and $\boldsymbol{U}_2$, together with the orientation $\rot{}{}$ and position $\pos{}{}$ are now completely and uniquely determined. The derivation of this $SU$-decomposition algorithm proves that a unique $SU$-decomposition exists under the derived regularity conditions. \hfill$\blacksquare$
\end{proof}

\begin{property}
	\label{pr:SU_invariance}
	The matrix $\boldsymbol{U}$ is invariant to changes in coordinate system.
\end{property}
\begin{proof}
Consider $\boldsymbol{\Xi}$ expressed in coordinate system $\{1\}$, denoted as ${_1}\boldsymbol{\Xi}$. 
The coordinate system of ${_1}\boldsymbol{\Xi}$ can be changed to a different coordinate system $\{2\}$ by left multiplication with the screw-transformation matrix ${^1_2}\boldsymbol{S}$, such that:
\begin{equation}
	\label{eq:frames}
	{_2}\boldsymbol{\Xi} = {^1_2}\boldsymbol{S} \ {_1}\boldsymbol{\Xi} .
\end{equation}
Substituting the $SU$-decomposition~\eqref{eq:bilts_SU} of ${_1}\boldsymbol{\Xi}$ in \eqref{eq:frames} results in:
\begin{equation}
	{_2}\boldsymbol{\Xi} = {^1_2}\boldsymbol{S} \ \big({_1}\boldsymbol{S} \ \boldsymbol{U}\big) .
\end{equation}
Since the matrix multiplication of the two screw-transformation matrices, ${^1_2}\boldsymbol{S}$ and ${_1}\boldsymbol{S}$, yields another screw-transformation matrix ${_2}\boldsymbol{S}$,  this results in:
\begin{equation}
	{_2}\boldsymbol{\Xi} = {_2}\boldsymbol{S} \ \boldsymbol{U} .
\end{equation}
This proves that ${_2}\boldsymbol{S} \ \boldsymbol{U}$ is a valid $SU$-decomposition of ${_2}\boldsymbol{\Xi}$. Moreover, since the $SU$-decomposition is unique, the decomposition ${_2}\boldsymbol{S} \ \boldsymbol{U}$ is the only possible $SU$-decomposition of ${_2}\boldsymbol{\Xi}$. Hence, the $SU$-decomposition of ${_2}\boldsymbol{\Xi}$ necessarily results in the same $\boldsymbol{U}$ matrix as the one for ${_1}\boldsymbol{\Xi}$. This proves that $\boldsymbol{U}$ is invariant to changes in the coordinate system of $\boldsymbol{\Xi}$. \hfill$\blacksquare$
\end{proof}
\begin{remark}
	The matrix $\boldsymbol{\Xi}$ is a $6 \times 3$ matrix and hence contains $18$ degrees of freedom (DoFs). In the $SU$-decomposition~\eqref{eq:bilts_SU}, these $18$ DoFs are redistributed between $\boldsymbol{S}$ and $\boldsymbol{U}$. Matrix $\boldsymbol{S}$ contains $6$ DoFs: three arising from the position vector $\boldsymbol{p}$ and three arising from the rotation matrix $\boldsymbol{R}$. The remaining $12$ DoFs are contained in $\boldsymbol{U}$. The variability of the entries of $\boldsymbol{\Xi}$ can hence be separated into variability due to (1) a change of coordinate system, captured entirely by $\boldsymbol{S}$, and (2) a change in the intrinsic evolution of the screw $\boldsymbol{\xi}[x_i]$, captured in an invariant manner up to second order by $\boldsymbol{U}$. Furthermore, since both $\boldsymbol{S}$ and $\boldsymbol{U}$ are unique in the regular case, the $SU$-decomposition is a \textit{bijective map} between $\boldsymbol{\Xi}$ and the pair of matrices $(\boldsymbol{S},\boldsymbol{U})$.
\end{remark}

\subsection{Geometric Interpretation of the $SU$-decomposition}
\label{sec:geom}

The matrix $\boldsymbol{R}$ and vector $\boldsymbol{p}$ can be interpreted as the orientation and position coordinates of an orthogonal coordinate system, functionally defined by the screw axis and its motion (see Fig.~\ref{fig:constructiuon_frame}). The first column $\boldsymbol{r}_1$ of $\boldsymbol{R}$ defines the $x$-axis of this functional coordinate system and is aligned with the screw axis of $\boldsymbol{\xi}[x_{i-1}]$.  The third column $\boldsymbol{r}_3$ of $\boldsymbol{R}$ defines the $z$-axis of this functional coordinate system and is aligned with the common normal of the screw axes of $\boldsymbol{\xi}[x_{i-1}]$ and $\boldsymbol{\xi}[x_{i}]$. The $y$-axis follows from right-handed orthogonality. The origin $\boldsymbol{p}$ coincides with the intersection point of the screw axis of $\boldsymbol{\xi}[x_{i-1}]$ and the common normal. Although this geometric interpretation can be intuitively inferred from the structure of $\boldsymbol{U}$, particularly its zero entries, proofs of this interpretation are provided in Appendix~\ref{app:proofs} for completeness.
\begin{figure}
	\centering
	\includegraphics[width=0.4\textwidth]{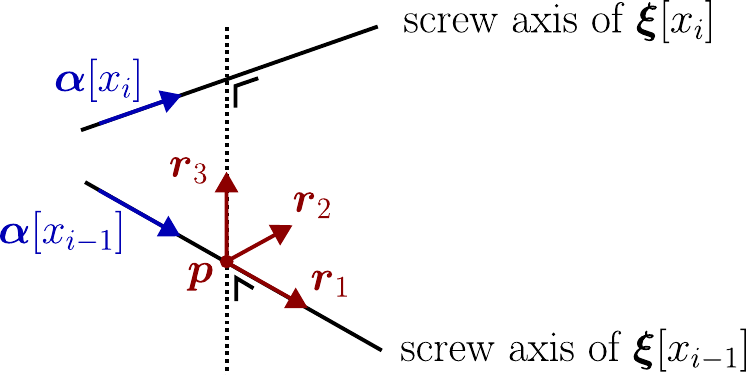}
	\caption{Geometric interpretation of $\boldsymbol{p}$ and $\boldsymbol{R} = [\boldsymbol{r}_1~\boldsymbol{r}_2~\boldsymbol{r}_3]$, which are the components of the screw-transformation matrix $\boldsymbol{S}$ in the $SU$-decomposition of \mbox{$\boldsymbol{\Xi} = [\boldsymbol{\xi}[x_{i-1}]~\boldsymbol{\xi}[x_i]~\boldsymbol{\xi}[x_{i+1}]]$}. The screw axes of $\boldsymbol{\xi}[x_{i-1}]$ and $\boldsymbol{\xi}[x_{i}]$ are shown as solid black lines. The common normal of the two screw axes is shown as a dotted black line.} 
	\label{fig:constructiuon_frame}
\end{figure}
\begin{remark}
	This geometric interpretation highlights that the $SU$-decomposition not only allows to compute the DUTIR of $\boldsymbol{\Xi}$, but it also enables an efficient, closed-form computation of key geometric entities. In particular, by computing $\boldsymbol{p}$ and $\boldsymbol{R}$, we compute the position and orientation of the screw axis of $\boldsymbol{\xi}[x_{i-1}]$ as well as the common normal of the screw axes of $\boldsymbol{\xi}[x_{i-1}]$ and $\boldsymbol{\xi}[x_{i}]$.
\end{remark}
\begin{remark}
	The position $\boldsymbol{p}$ and matrix $\boldsymbol{R}$ are geometric entities anchored to the screw trajectory. As a result, they transform consistently under changes of the coordinate system. Consequently, the computation of $\boldsymbol{R}$ and $\boldsymbol{p}$ from $\boldsymbol{\Xi}$ constitutes an \textit{equivariant map}~\cite{equivariant} with respect to coordinate system transformations.
\end{remark}

\subsection{Regularization of the $SU$-decomposition}
\label{sec:regularization}
When the regularity conditions~\eqref{eq:SU_reg_con_1} are not met, the QR-decomposition of $\boldsymbol{A}$ is ill-defined and the system of equations \eqref{eq:extendedqr_scalar0}-\eqref{eq:extendedqr_scalar} becomes underdetermined. Consequently, an infinite number of possible solutions for $\boldsymbol{R}$, $\boldsymbol{p}^*$, and $\boldsymbol{p}$ exist. 
An intuitive example is the case of a pure translational motion (where $\boldsymbol{\alpha} =  \boldsymbol{\omega} = \boldsymbol{0}$ and $\boldsymbol{\beta} = \boldsymbol{v} \neq 0$). In this case, the screw axes of $\boldsymbol{\xi}[x_{i-1}]$ and $\boldsymbol{\xi}[x_{i}]$ are ill-defined, since a pure translation can be represented by (1) a screw axis at a finite distance from the moving body and parallel to the translational velocity $\boldsymbol{v}$, or by (2) a screw axis at an infinite distance from the body and perpendicular to the translational velocity $\boldsymbol{v}$, since a pure translation can also be represented as a pure rotation (with infinitely small magnitude) along a screw axis at an infinite distance. In practice, the second solution is rarely meaningful. 
Inspired by the first, more intuitive case, we introduce a regularization procedure for the $SU$-decomposition in this subsection. 

\textbf{1) Regularization of $\boldsymbol{p}$:}
To regularize $\boldsymbol{p}$, we first predefine a maximum allowed distance between the point on the screw axis $\boldsymbol{p}$ and the origin of the coordinate system and denote it as $L$. This value for $L$ can be interpreted as a \textit{geometric scale}, defining a spatial region in which the computed value of $\boldsymbol{p}$ is considered relevant.
Then, after solving \eqref{eq:extendedqr_scalar0} - \eqref{eq:extendedqr_scalar} to obtain $\pos{}{}^*$, we evaluate whether regularization is required, which is the case when $\Vert\boldsymbol{p}^*\Vert > L$. In that case, we search for a regularized version of $\boldsymbol{p}^*$, denoted as $\hat{\boldsymbol{p}}^*$, that satisfies $\Vert\hat{\boldsymbol{p}}^*\Vert = L$. After performing the regularization action (which is detailed later in this subsection), the regularized matrix $\hat{\boldsymbol{U}}_{2}$ can be found by solving \eqref{eq:R2-computation} using $\hat{\boldsymbol{p}}^*$. 
An important consequence of this regularization is that, when active, the regularized matrix $\hat{\boldsymbol{U}}_2$ will no longer be strictly upper-triangular. Instead, the lower-diagonal elements of $\hat{\boldsymbol{U}}_2$ will be non-zero and will depend on the location of the origin of the coordinate system.
Ideally, these lower-diagonal elements remain small such that a high resemblance to the original $\boldsymbol{U}_2$ is obtained. To this end, we design the regularization such that the sum of squares of the lower-diagonal terms in the first column\footnote{Alternatively, the sum of squares of all the lower-diagonal terms in $\hat{\boldsymbol{U}}_2$ can be minimized, although this will result in a problem with a higher complexity.} of $\hat{\boldsymbol{U}}_2$ are minimized. Expressions for these terms, denoted as $\epsilon_{51}$ and $\epsilon_{61}$, follow from \eqref{eq:extendedqr_scalar0} and \eqref{eq:extendedqr_scalar1}:
\begin{equation}
	\epsilon_{51} = (\rot{}{}^\top \boldsymbol{B})_{21} -u_{11} \hspace{1pt} \hat{p}^*_z \quad \text{and} \quad
	\epsilon_{61} = (\rot{}{}^\top \boldsymbol{B})_{31} + u_{11} \hspace{1pt} \hat{p}^*_y ,
\end{equation}
where $\epsilon_{ij}$ denotes the lower-diagonal entry at the $i$-th row and $j$-th column of $\boldsymbol{U}$. 
Minimizing the sum of squares of $\epsilon_{51}$ and $\epsilon_{61}$ while enforcing $\Vert \hat{\boldsymbol{p}}^* \Vert = L$ can be done by minimizing the Lagrangian function $\mathcal{L}(\hat{p}^*_x,\hat{p}^*_y,\hat{p}^*_z,\lambda)$, defined as:
\begin{equation}
	\label{eq:lagrangian}
	\mathcal{L}(\hat{p}^*_x,\hat{p}^*_y,\hat{p}^*_z,\lambda) = \epsilon_{51}^2 + \epsilon_{61}^2 + \lambda \left[L^2 - (\hat{p}^*_x)^2 - (\hat{p}^*_y)^2 - (\hat{p}^*_z)^2\right],
\end{equation}
where $\lambda$ denotes the Lagrange multiplier enforcing the constraint.
Extrema of the Lagrangian function can be found by searching for points where its gradient is equal to the zero vector. Following this approach, two solutions for $\hat{\boldsymbol{p}}^*$ can be found. 
The first solution has the following form: 
\begin{equation} 
	\label{eq:sol_px1}
	\hat{p}^*_y = p^*_y \quad , \quad \hat{p}^*_z = p^*_z  \quad \text{and} \quad
	\hat{p}^*_x =  \text{sign}(p^*_x)\sqrt{L^2 - (p^*_y)^2 - (p^*_z)^2}. 
\end{equation}
Hence, the first solution consists of retaining the original $p^*_y$ and $p^*_z$, while projecting $p^*_x$ along the screw axis of $\boldsymbol{\xi}[x_{i-1}]$ onto the spherical manifold defined by $L$. This solution is a valid solution when $(p^*_y)^2 + (p^*_z)^2 < L^2$.  

The second solution has the following form:
\begin{equation}
	\label{eq:sol_2}
	\hat{p}^*_x = 0
	\quad , \quad 
	\hat{p}^*_y  = \small L \frac{p^*_y}{\sqrt{(p^*_y)^2 + (p^*_z)^2}} \quad \text{and} \quad 
	\hat{p}^*_z  = \small L \frac{p^*_z}{\sqrt{(p^*_y)^2 + (p^*_z)^2}} .
\end{equation}
Hence, the second solution consists of setting $\hat{p}^*_x = 0$, while isotropically projecting $p^*_y$ and $p^*_z$ onto the spherical manifold defined by $L$. The two solutions for $\hat{\boldsymbol{p}}^*$ are visualized in Figure~\ref{fig:regularization}. 
\begin{figure}
	\centering
	\includegraphics[width=0.7\textwidth]{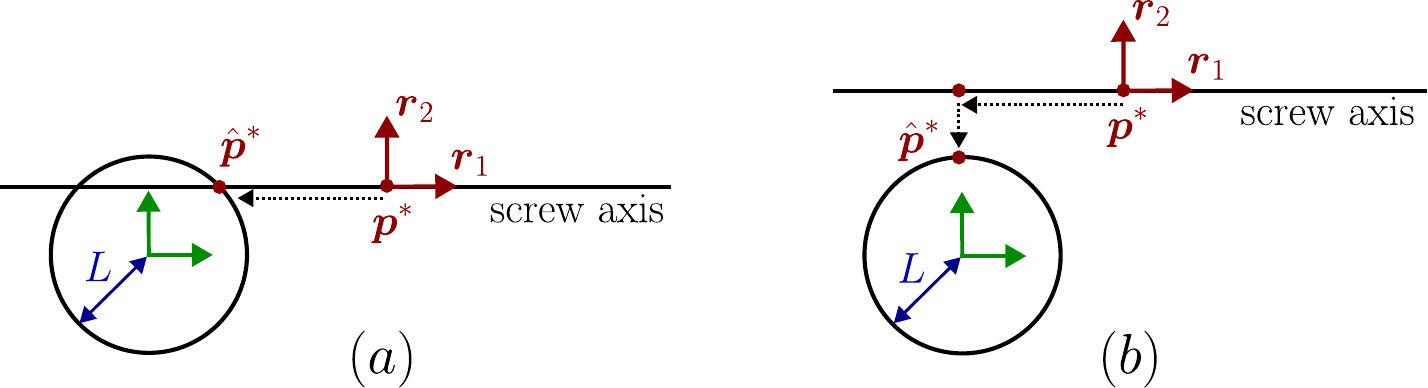}
	\caption{Two-dimensional illustration of the regularization of $\boldsymbol{p}^*$, showing the coordinate system of $\boldsymbol{\Xi}$ (green), the spherical manifold defined by $L$ (black circle), the screw axis of $\boldsymbol{\xi}[x_{i-1}]$ (black line), and the functional coordinate system defined by $\boldsymbol{p}^*$, $\boldsymbol{r}_1$, and $\boldsymbol{r}_2$ (red). The solution for $\hat{\boldsymbol{p}}^*$ is illustrated for (a) the case where the screw axis intersects the spherical manifold, i.e., when $(p^*_y)^2 + (p^*_z)^2 < L^2$, and (b) the case where the screw axis does not intersect the spherical manifold, i.e., when $(p^*_y)^2 + (p^*_z)^2 > L^2$.}
	\label{fig:regularization}
	\vspace{-2pt}
\end{figure}

\noindent Based on these two solutions, the regularization of $\boldsymbol{p}$ is implemented as follows. We first compute $\boldsymbol{p}^*$ and test whether regularization is required, which is the case when $\Vert \boldsymbol{p}^*\Vert > L$. If so, we test whether $(p^*_y)^2 + (p^*_z)^2 < L^2$. If this condition holds, it means that the screw axis intersects the spherical manifold (see Fig.~\ref{fig:regularization}a) and the first solution for $\hat{\boldsymbol{p}}^*$ should be selected. If the condition does not hold, we use the second solution. Finally, $\hat{\boldsymbol{p}}$ is found as $\hat{\boldsymbol{p}} = \boldsymbol{R} \  \hat{\boldsymbol{p}}^*$, and the matrix $\hat{\boldsymbol{U}}_{2}$ can be found by solving \eqref{eq:R2-computation} using $\hat{\boldsymbol{p}}^*$.

\textbf{2) Regularization of $\boldsymbol{R}$:}
In \eqref{eq:A1}, $\boldsymbol{R}$ is computed by triangularizing $\boldsymbol{A}$ into $\boldsymbol{U}_1$. However, when the regularity conditions~\eqref{eq:SU_reg_con_1} are not met, $\boldsymbol{U}_1$ does not provide sufficient information to determine a unique rotation matrix $\boldsymbol{R}$. To address this issue, we propose to regularize $\boldsymbol{R}$ by using additional information about the trajectory encoded in $\hat{\boldsymbol{U}}_2$. For instance, in the case of pure translational motion, determining $\boldsymbol{R}$ from $\boldsymbol{A}$ is ill-posed since $\boldsymbol{A}$ consists of rotational velocity vectors, which are zero vectors in this scenario. In this case, the columns of $\boldsymbol{R}$ should be inferred from the translational velocity vectors in $\hat{\boldsymbol{U}}_2$.

In particular, we regularize $\boldsymbol{R}$ as $\hat{\boldsymbol{R}} = \boldsymbol{R}\boldsymbol{R}_c^\top$, where $\boldsymbol{R}_c$ is a corrective rotation matrix. The rotation $\boldsymbol{R}_c$ is determined by jointly optimizing over the column spaces of $\boldsymbol{U}_1$ and $\hat{\boldsymbol{U}}_2$ such that both $\boldsymbol{R}_c \boldsymbol{U}_1$ and $\boldsymbol{R}_c \hat{\boldsymbol{U}}_2$ are as close as possible to upper-triangular matrices. This is achieved via a Procrustes alignment~\cite{schonemann1966generalized}:
\begin{equation}
	\label{eq:reg_R}
	\operatorname*{minimize}_{\boldsymbol{R}_c} \; ~~~w^2\Big\Vert \boldsymbol{U}_1 - \boldsymbol{R}_{c}\boldsymbol{U}_1 \Big\Vert^2_F ~+~ \left\Vert\hat{\boldsymbol{U}}_{2,\triangle} - \boldsymbol{R}_{c}\hat{\boldsymbol{U}}_2 \right\Vert^2_F , 
\end{equation}
where $\Vert \cdot \Vert_F$ denotes the matrix Frobenius norm, $\hat{\boldsymbol{U}}_{2,\triangle}$ is the triangularization\footnote{Following the same procedure used to triangularize $\boldsymbol{A}$ into $\boldsymbol{U}_1$.} of the approximately upper-triangular matrix $\hat{\boldsymbol{U}}_{2}$, and $w$ is a weighting parameter expressed in meters. Introducing the weight $w$ is necessary because the entries of $\boldsymbol{U}_1$ and $\hat{\boldsymbol{U}}_2$ have different physical units: for twists, $\boldsymbol{U}_1$ contains angular components in \textit{rad/progress}, while $\hat{\boldsymbol{U}}_2$ contains translational components in \textit{m/progress}. For wrenches, the corresponding units are \textit{N} and \textit{Nm}, respectively. 

Given the computed matrix $\boldsymbol{R}_c$, the regularized matrix $\hat{\boldsymbol{U}}_1$ is obtained as $\hat{\boldsymbol{U}}_1 = \boldsymbol{R}_{c}\boldsymbol{U}_1$, while $\hat{\boldsymbol{U}}_2$ is updated as $ \boldsymbol{R}_{c}\hat{\boldsymbol{U}}_2$. 

\begin{remark}
	The proposed regularization introduces two distinct solution cases for the regularized $SU$-decomposition. When the regularization is inactive, the solutions for $\hat{\boldsymbol{p}}$ and $\hat{\boldsymbol{R}}$ remain \textit{purely coordinate-invariant}. When the regularization is active, the solutions for $\hat{\boldsymbol{p}}$ and $\hat{\boldsymbol{R}}$ become sensitive to the choice of the origin of the coordinate system. In this case, decreasing $L$ results in an increase in sensitivity to this choice. In other words, decreasing $L$ shifts the solution for the $SU$-decomposition from a purely coordinate invariant one to a coordinate-dependent one anchored to the origin of the coordinate system. 
	
	We propose to set $w = L$ such that the regularization is completely determined by a single parameter. Choosing $w = L$ is allowed since $w$ and $L$ share the same physical unit (meters). Furthermore, this choice is conceptually well motivated: both $L$ and $w$ govern the extent to which the regularized $SU$-decomposition depends on the coordinate-system origin. Specifically, decreasing either $L$ or $w$ increases the sensitivity to this origin.
	
	Finally, in practical applications, we propose to treat $L$ as a hyperparameter. For instance, in recognition tasks, an optimal value for $L$ for a given dataset can be selected via standard hyperparameter validation procedures, such as maximizing recognition performance on a held-out validation set. 
\end{remark}

\subsection{Numerical Example}
\label{sec:examples}

In this subsection, we present a numerical example demonstrating the application of the $SU$-decomposition to rigid-body motion trajectory data. We use real trajectory data from \cite{verduyn2023enhancingmotiontrajectorysegmentation}, corresponding to a human-demonstrated pouring task performed with a kettle. The trajectory is parametrized with respect to time $t$, and consists of three distinct motion segments: sliding the kettle along the table ($t\in[0,0.7]$), lifting the kettle ($t\in[0.7,1.6]$), and pouring ($t\in[1.6,2.8]$). 

To highlight coordinate invariance, the $SU$-decomposition is applied twice using different choices of the coordinate system attached to the body (see Fig.~\ref{fig:overview}a). From the successive positions and orientations of the coordinate system, twist coordinates are computed via numerical differentiation using the logarithmic map~\cite{lynch2017modern} (see Fig.~\ref{fig:overview}b). As expected, the resulting twist coordinates differ under changes in coordinate system.  Subsequently, $\boldsymbol{U}$ is evaluated along the trajectory at successive time instances. To assess the effect of regularization, we computed the $SU$-decomposition both without regularization (see Fig.~\ref{fig:overview}c) and with regularization when choosing $L = 0.3$m (see Fig.~\ref{fig:overview}d). 

In the unregularized case (see Fig.~\ref{fig:overview}c), the computed $\boldsymbol{U}$ is \textit{strictly} twice upper-triangular. Furthermore, the results for the first coordinate system are identical to the ones for the second coordinate system, which confirms the coordinate invariance of $\boldsymbol{U}$. However, the components $u_{53}$ and $u_{63}$ exhibit pronounced nervousness during the sliding segment ($t\in[0,0.7]$), reflecting the nervous motion of the screw axis during this segment. This nervousness arises since the sliding motion corresponds to a near-pure translation, and pure translation corresponds to a singularity in the $SU$-decomposition.

In the regularized case (see Fig.~\ref{fig:overview}d), the computed $\boldsymbol{U}$ is \textit{approximately} twice upper-triangular. Figure~\ref{fig:overview}d also shows that the nervousness of $u_{53}$ and $u_{63}$ is substantially reduced, and that more pronounced signals appear in the profiles of $u_{41}$, $u_{42}$, and $u_{43}$ during the sliding segment. This is a desirable effect since the magnitude of these signals corresponds to the magnitude of the translational velocity during sliding. These results hence showcase that the regularized approach yields a less nervous and more interpretable representation of near-pure translational motion.

Additional examples and the implementation of the $SU$-decomposition are provided at \url{https://github.com/arnoverduyn/SU_decomposition}.

\begin{figure}[t!]
	\centering
	\includegraphics[width=0.85\textwidth]{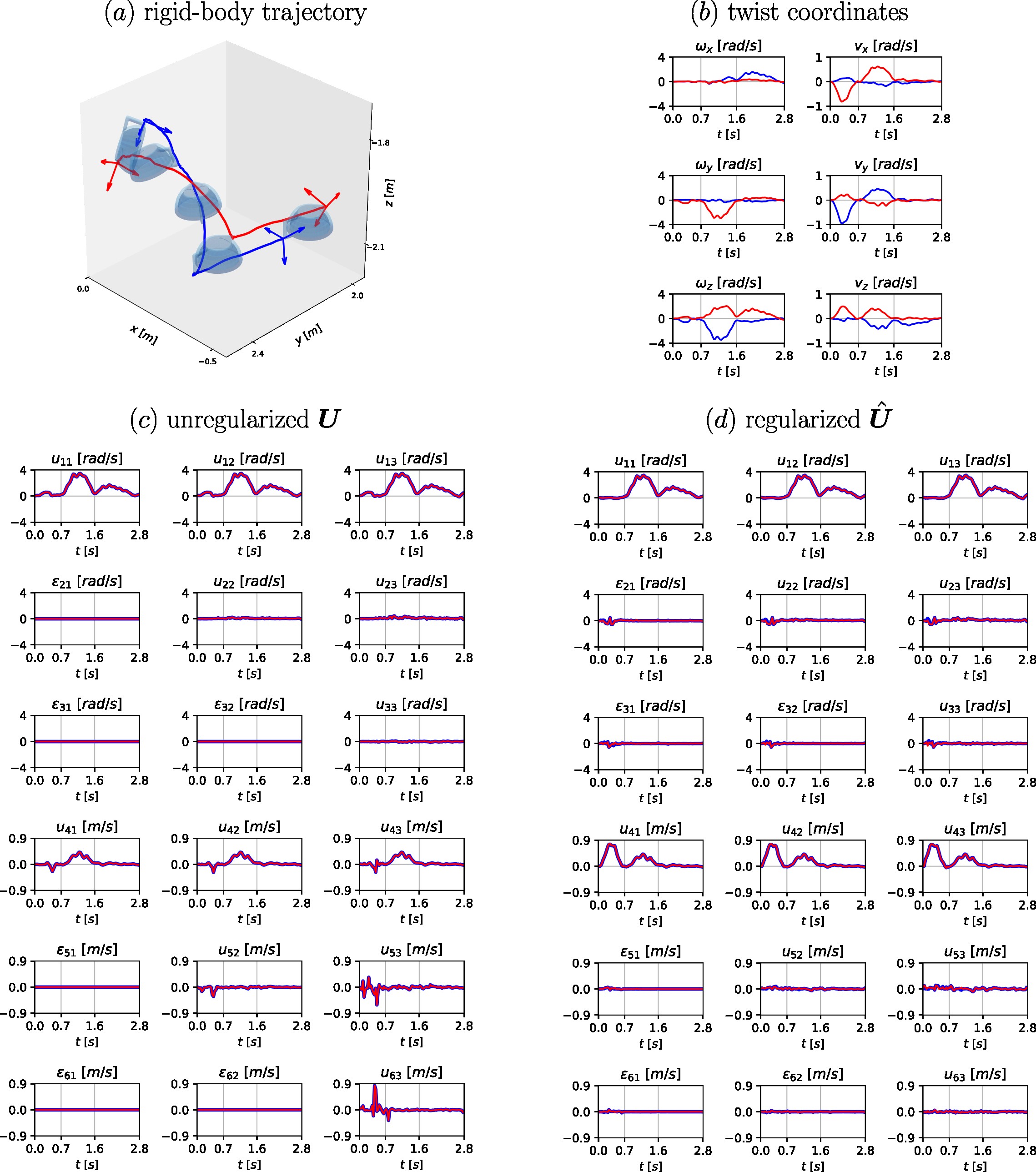}
	\caption{Numerical example of the $SU$-decomposition applied to rigid-body trajectory data. 
	(a) Visualization of the trajectory for two different choices (red and blue) of the coordinate system attached to the body; (b) Twist coordinates extracted from the trajectory via numerical differentiation using the logarithmic map~\cite{lynch2017modern}; (c) Components of the invariant representation $\boldsymbol{U}$ obtained without regularization. In this case, the lower-diagonal entries $\varepsilon_{ij}$ of $\boldsymbol{U}$ are exactly zero; (d) Components of $\hat{\boldsymbol{U}}$ obtained with regularization. In this case, the lower-diagonal entries $\varepsilon_{ij}$ of $\boldsymbol{U}$ are approximately zero.} 
	\label{fig:overview}
	\vspace{-5pt}
\end{figure}
 
\section{Discussion and Conclusion}
\label{sec:discussion}
The goal of this paper was to present a novel coordinate-invariant representation of screw trajectories while ensuring (1) pure locality of the representation, (2) low computational complexity, and (3) high robustness. This paper successfully achieved this goal, which is substantiated by the following contributions. We introduced the $SU$-decomposition and the corresponding DUTIR of $\boldsymbol{\Xi}$. The matrix $\boldsymbol{\Xi}$ captures the local second-order evolution of the screw trajectory, while the DUTIR, $\boldsymbol{U}$, provides a coordinate-invariant representation of this evolution. The proposed $SU$-decomposition is computationally efficient, since it involves the computation of analytical, closed-form solutions without iterative procedures. Furthermore, we introduced regularization to enhance the robustness of the $SU$-decomposition near irregularity. 

The proposed $SU$-decomposition and regularization have been experimentally validated for motion trajectory data in prior works. Specifically, inspired by this decomposition, the work in~\cite{phdarno} derived a coordinate-invariant similarity measure between two rigid-body trajectories. The derived invariant similarity measure was then experimentally validated for rigid-body motion recognition using three distinct rigid-body motion datasets. 
Despite substantial coordinate system variations, the approach achieved a high average recognition rate of 96.5\% across the different datasets, outperforming other similarity measures based on existing invariant trajectory representations~\cite{DeSchutterJoris2010,Vochten2015,vochten2023invariant,DSRF2018,RRV2018,Lee2018}.  

The recognition experiments in~\cite{phdarno} also showed that the proposed regularization significantly improved the recognition robustness when dealing with singular motions, such as pure translations. Although the proposed regularization reintroduces a sensitivity to the choice of the origin of the coordinate system, this sensitivity has shown to remain minor in~\cite{phdarno}. This is because, by design, the regularization is only activated near singular cases and furthermore, the sensitivity of the invariant representation to variations in coordinate-system-origin is limited by minimizing the lower-diagonal terms in the first column of $\hat{\boldsymbol{U}}_2$. 
 
In other prior work~\cite{HPM_recognition}, a real-time recognition system was developed inspired by the proposed $SU$-decomposition. This system robustly recognized gestures (i.e. motion signatures traced by the palm of the human hand) across varying coordinate systems in real time, achieving a solid $F_1$-score of 92.3\%. 

The discussed prior works focused on the motion trajectories of rigid bodies. Importantly, the $SU$-decomposition and DUTIR are also directly applicable to the trajectories of interaction forces (i.e. through the wrench). The application to wrenches is for example relevant for identifying object-manipulation tasks in contact with the environment. Incorporating wrench information is then beneficial if the aim is to discriminate motions that exhibit similar motion profiles but differ significantly in their contact wrench profiles. Experimental validation of the proposed methods for this type of task is part of future work. 

However, invariant representations have their limitations and hence should not be employed in every practical scenario. For example, in scenarios where coordinate system calibration is straightforward, reliable, and the trajectories are consistently performed relative to the calibrated coordinate system (e.g. performed at the same location, in the same direction, etc.), then traditional coordinate-dependent methods may outperform invariant methods. This is because invariant representations, by design, discard information tied to the coordinate system. Nevertheless, invariant representations are particularly valuable for analysing trajectories when coordinate systems are misaligned with their exact relationships unknown, when the relation between motion direction and coordinate axes is non-deterministic, or when aiming to eliminate the need for coordinate system calibration when such calibration is labour-intensive.

In conclusion, the proposed DUTIR and $SU$-decomposition offer a promising way for future algorithms in robotics that need to compare trajectories in an invariant manner without compromising robustness or computational efficiency.

\begin{credits}

\subsubsection{\ackname} This work was supported by a project that has received funding from the European Research Council (ERC) under the European Union's Horizon 2020 Research and Innovation Programme (Grant agreement No. 788298).

\subsubsection{\discintname}
The authors have no competing interests to declare that are relevant to the content of this article. 
\end{credits}
%
%

\newpage
\appendix

\section{Interpretation of the $SU$-decomposition: Proofs}
\phantomsection
\label{app:proofs}

In Section~\ref{sec:geom}, it is stated that:
\begin{enumerate}
	\item The first column $\boldsymbol{r}_1$ of $\boldsymbol{R}$ is parallel to the screw axis of $\boldsymbol{\xi}[x_{i-1}]$;
	\item The third column $\boldsymbol{r}_3$ of $\boldsymbol{R}$ is parallel to the common normal of the screw axes of $\boldsymbol{\xi}[x_{i-1}]$ and $\boldsymbol{\xi}[x_{i}]$;
	\item The vector $\boldsymbol{p}$ represents a point on the screw axis of $\boldsymbol{\xi}[x_{i-1}]$;
	\item The vector $\boldsymbol{p}$ represents the intersection point between the screw axis of $\boldsymbol{\xi}[x_{i-1}]$ and the common normal of the screw axes of $\boldsymbol{\xi}[x_{i-1}]$ and $\boldsymbol{\xi}[x_{i}]$.
\end{enumerate} 
In this appendix, we present proofs of these statements. Before presenting the proofs, we first summarize a set of relations involving $\boldsymbol{\alpha}[x_{i-1}]$, $\boldsymbol{\alpha}[x_{i}]$, $\boldsymbol{r}_1$, $\boldsymbol{r}_2$, $u_{11}$, $u_{12}$, and $u_{22}$. The vectors $\boldsymbol{r}_1$ and $\boldsymbol{r}_2$ are obtained by orthonormalizing $\boldsymbol{\alpha}[x_{i-1}]$ and $\boldsymbol{\alpha}[x_{i}]$. The scalars $u_{11}$, $u_{12}$, and $u_{22}$ represent the components of $\boldsymbol{\alpha}[x_{i-1}]$ and $\boldsymbol{\alpha}[x_{i}]$ along $\boldsymbol{r}_1$ and $\boldsymbol{r}_2$ (see Fig.~\ref{fig:frame_orientation}).
\begin{figure}
	\centering
	\includegraphics[width=0.22\textwidth]{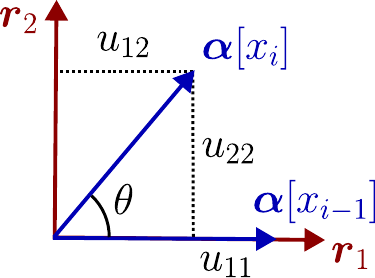}
	\caption{Geometric relations between $\boldsymbol{\alpha}[x_{i-1}]$, $\boldsymbol{\alpha}[x_{i}]$, $u_{11}$, $u_{12}$, and $u_{22}$.} 
	\label{fig:frame_orientation}
\end{figure}

By applying trigonometric identities, the following relations can be derived:
\begin{align}
	\label{eq:list1}
	\Vert \boldsymbol{\alpha}[x_{i-1}] \Vert &= u_{11}, \\
	\label{eq:list2}
	\boldsymbol{\alpha}[x_{i-1}]\cdot\boldsymbol{\alpha}[x_{i}] &= \Vert \boldsymbol{\alpha}[x_{i-1}] \Vert \Vert \boldsymbol{\alpha}[x_{i}] \Vert \cos \theta = u_{11}u_{12}, \\
	\label{eq:list3}
	\boldsymbol{\alpha}[x_{i-1}]\times\boldsymbol{\alpha}[x_{i}] &= \Vert \boldsymbol{\alpha}[x_{i-1}] \Vert \Vert \boldsymbol{\alpha}[x_{i}] \Vert \sin \theta ~\boldsymbol{r}_3 = u_{11}u_{22}\boldsymbol{r}_3,
\end{align}
and:
\begin{equation}
	\label{eq:list4}
	\small
	\Vert \boldsymbol{\alpha}[x_{i-1}] \Vert^2 \Vert \boldsymbol{\alpha}[x_{i}] \Vert^2 - \left(\boldsymbol{\alpha}[x_{i-1}]\cdot\boldsymbol{\alpha}[x_{i}]\right)^2 = u_{11}^2 \Vert \boldsymbol{\alpha}[x_{i}] \Vert^2 \left(1-\cos^2 \theta\right) = u_{11}^2u_{22}^2. 
\end{equation}
Next, we present the proofs of the four enumerated statements.

\textit{Proof of 1}: The proof follows from \eqref{eq:direction_screw} and \eqref{eq:r1}. Equation \eqref{eq:direction_screw} shows that the direction of the screw axis of $\boldsymbol{\xi}[x_{i-1}]$ is obtained by normalizing the vector $\boldsymbol{\alpha}[x_{i-1}]$, while \eqref{eq:r1} confirms that $\boldsymbol{r}_1$ represents this normalized vector.  \hfill$\blacksquare$

\textit{Proof of 2}: The proof follows from \eqref{eq:list3}, which shows that  $\boldsymbol{r}_3$ lies parallel to the vector $\boldsymbol{\alpha}[x_{i-1}] \times \boldsymbol{\alpha}[x_i]$. This vector is orthogonal to the screw axes of $\boldsymbol{\xi}[x_{i-1}]$ and $\boldsymbol{\xi}[x_{i}]$, and thus defines the direction of their common normal. \hfill $\blacksquare$

\textit{Proof of 3}: The point $\boldsymbol{p}$ lies on the screw axis of $\boldsymbol{\xi}[x_{i-1}]$ if and only if the following parametric line equation is satisfied:
	\begin{align}
		\label{eq:line_eq_screw_axis0}
		\boldsymbol{p} &=  \boldsymbol{p}_\perp[x_{i-1}]+ p_\parallel  \boldsymbol{e} .
	\end{align}
	The first term in \eqref{eq:line_eq_screw_axis0} represents the point on the screw axis of $\boldsymbol{\xi}[x_{i-1}]$ closest to the origin of the coordinate system (see \eqref{eq:p_perp}). The second term in \eqref{eq:line_eq_screw_axis0} represents a displacement along the screw axis of $\boldsymbol{\xi}[x_{i-1}]$, with $p_\parallel\in \mathbb{R}$ a free parameter.
	
	By substituting \eqref{eq:direction_screw} and \eqref{eq:p_perp} in \eqref{eq:line_eq_screw_axis0}, then left multiplying the result with $\boldsymbol{R}^\top$, and finally using $\boldsymbol{p}^*=\boldsymbol{R}^\top\boldsymbol{p}$ and \eqref{eq:list1}, it can be shown that~\eqref{eq:line_eq_screw_axis0} reduces to:

	\begin{equation}
		\boldsymbol{p}^* = \small \frac{1}{u_{11}} \begin{pmatrix}
			0 \\ -(\boldsymbol{R}^\top \boldsymbol{B})_{31} \\
			(\boldsymbol{R}^\top \boldsymbol{B})_{21}
		\end{pmatrix} + p_\parallel \begin{pmatrix}
			1 \\ 0 \\ 0
		\end{pmatrix},
	\end{equation}
	which leads to the following three conditions:
	\begin{equation}
		\label{eq:relations_p}
		p^*_x = p_\parallel , \quad
		u_{11} p^*_y = - (\boldsymbol{R}^\top \boldsymbol{B})_{31} \quad \text{and} \quad
		u_{11} p^*_z = (\boldsymbol{R}^\top \boldsymbol{B})_{21} .
	\end{equation}
	The first condition is satisfied, since $p_\parallel$ is a free parameter. 
	The other two conditions are also satisfied, since they correspond to the expressions for computing $p^*_z$ and $p^*_y$ shown in \eqref{eq:extendedqr_scalar0} and \eqref{eq:extendedqr_scalar1}, respectively. This proves that the line equation in \eqref{eq:line_eq_screw_axis0} is satisfied, and hence that $\boldsymbol{p}$ lies on the screw axis of $\boldsymbol{\xi}[x_{i-1}]$.  \hfill$\blacksquare$

\textit{Proof of 4}: The common normal of two axes is the line that perpendicularly intersects both axes. The distance between the two intersection points is the shortest distance between the two axes. These two intersection points can hence be found by identifying the points on each axis with minimum mutual distance.
	
	Similarly to~\eqref{eq:line_eq_screw_axis0}, we can define a line equation for the screw axis of the second screw $\boldsymbol{\xi}[x_i]$. We denote the entities corresponding to the first and second screw, $\boldsymbol{\xi}[x_{i-1}]$ and $\boldsymbol{\xi}[x_i]$, with the subscripts $1$ and $2$, respectively, resulting in:
	\begin{align}
		\boldsymbol{p}_1 =  \boldsymbol{p}_{\perp1} + p_{\parallel1}~  \boldsymbol{e}_1 \quad \text{and} \quad \boldsymbol{p}_2 =  \boldsymbol{p}_{\perp2}+ p_{\parallel2}~ \boldsymbol{e}_2 .
	\end{align}
	The points on the two screw axes with minimal mutual distance can be found by minimizing the squared distance, \mbox{$d^2 = \Vert \boldsymbol{p}_1- \boldsymbol{p}_2\Vert^2$}. The minimum can be found by searching for the points where the gradient, $\boldsymbol{\nabla}d^2 = (\frac{\partial d^2}{\partial p_{\parallel1}}, \frac{\partial d^2}{\partial p_{\parallel2}})$, is equal to the zero vector. This results in a system of two equations in the variables $p_{\parallel1}$ and $p_{\parallel2}$. Following this approach, the following expression for $p_{\parallel1}$ can be found:
	\begin{equation}
		\label{eq:p_parallel_long}
		p_{\parallel1}=\small\frac{\Vert\boldsymbol{\alpha}[x_{i-1}]\Vert\left[\boldsymbol{\alpha}[x_{i-1}]\cdot\left(\boldsymbol{\alpha}[x_{i}]\times\boldsymbol{\beta}[x_{i}]\right)\right]+\frac{\boldsymbol{\alpha}[x_{i-1}]\cdot\boldsymbol{\alpha}[x_{i}]}{\Vert\boldsymbol{\alpha}[x_{i-1}]\Vert}\left[\boldsymbol{\alpha}[x_{i}]\cdot\left(\boldsymbol{\alpha}[x_{i-1}]\times\boldsymbol{\beta}[x_{i-1}]\right)\right]}{\Vert \boldsymbol{\alpha}[x_{i-1}]\Vert^2 \Vert \boldsymbol{\alpha}[x_{i}]\Vert^2 - \left(\boldsymbol{\alpha}[x_{i-1}] \cdot \boldsymbol{\alpha}[x_{i}]\right)^2} .
	\end{equation}
	Using \eqref{eq:list1}, \eqref{eq:list2}, and \eqref{eq:list4}, expression \eqref{eq:p_parallel_long} can be rewritten as:
	\begin{equation}
		\label{eq:p_parallel_reduced}
		p_{\parallel1}=\small\frac{u_{11} \boldsymbol{\alpha}[x_{i-1}]\cdot\left(\boldsymbol{\alpha}[x_{i}]\times\boldsymbol{\beta}[x_{i}]\right)+u_{12}\boldsymbol{\alpha}[x_{i}]\cdot\left(\boldsymbol{\alpha}[x_{i-1}]\times\boldsymbol{\beta}[x_{i-1}]\right)}{u_{11}^2u_{22}^2} .
	\end{equation}
	Using the property: $\boldsymbol{a}\cdot\left(\boldsymbol{b} \times \boldsymbol{c}\right) = \boldsymbol{c}\cdot\left(\boldsymbol{a} \times \boldsymbol{b}\right)$, and relation~\eqref{eq:list3}, \eqref{eq:p_parallel_reduced} reduces to:
	\begin{equation}
		p_{\parallel1}=\small\frac{u_{11}^2u_{22} \boldsymbol{\beta}[x_{i}]\cdot\boldsymbol{r}_3-u_{11}u_{12}u_{22}\boldsymbol{\beta}[x_{i-1}]\cdot\boldsymbol{r}_3}{u_{11}^2u_{22}^2},
	\end{equation}
	which finally can be rewritten as:
	\begin{equation}
		p_{\parallel1}=\small\frac{1}{u_{22}}\left[( \boldsymbol{R}^\top\boldsymbol{B})_{32}-\frac{u_{11}}{u_{12}}( \boldsymbol{R}^\top\boldsymbol{B})_{31} \right].
	\end{equation}
	From \eqref{eq:extendedqr_scalar1} and \eqref{eq:extendedqr_scalar}, it follows that $p_{\parallel1} = p^*_x$, which proves the argument. \hfill $\blacksquare$

\bibliographystyle{splncs04}
\bibliography{mybibliography}

\end{document}